\documentclass{article}
\usepackage[utf8]{inputenc}
\usepackage[tbtags]{amsmath}
\usepackage{amsfonts,amssymb,mathrsfs,amscd,comment,amsthm}
\usepackage{natbib}

\usepackage{amsfonts}       
\usepackage{nicefrac}       
\usepackage{microtype}      
\usepackage{comment}
\usepackage{amssymb, amsmath, latexsym}

\usepackage{nicefrac}       

\newtheorem{theorem}{Theorem}

\begin{document}
\title{\bf\large On the Inverse Flow Matching Problem\\in the One-Dimensional and Gaussian Cases}
\author{Korotin A.\footnote{Applied AI Institute, iamalexkorotin@gmail.com}, Pammer G.\footnote{Graz University of Technology, gudmund.pammer@tugraz.at}}
\date{}

\maketitle

\vspace{-6mm}\begin{abstract}
    This paper studies the inverse problem of flow matching (FM) between distributions with finite exponential moment, a problem motivated by modern generative AI applications such as the distillation of flow matching models. Uniqueness of the solution is established in two cases – the one-dimensional setting and the Gaussian case. The general multidimensional problem remains open for future studies.
\end{abstract}

\vspace{2mm}We study the inverse flow matching problem (FM~\cite{liu2023flow}) between distributions $p_0, p_1 \in \mathcal{P}_{\rm exp}(\mathbb{R}^D)$ on $\mathbb{R}^D$ with a finite exponential moment, i.e., such that
$$\exists \lambda>0: \int \exp(\lambda \|x\|_2)d[p_0\!+\!p_1](x)<\infty.$$
For this problem, we establish the uniqueness of its solution in two cases: the one-dimensional ($D=1$) and the Gaussian.

We start by describing the \underline{\textbf{forward problem}} of FM \cite{liu2023flow}. Let $\pi\in \Pi(p_0,p_1)$ be a transport plan between $p_0$ and $p_1$, i.e., a joint distribution on $\mathbb{R}^{D}\times\mathbb{R}^{D}$ whose marginals are $p_0$ and $p_1$, respectively. Let $(X_0,X_1)\sim \pi$ be some random variable taking values in $\mathbb{R}^{D}\times\mathbb{R}^{D}$ with distribution $\text{Law}\big((X_0,X_1)\big)=\pi$. For each $t\in(0,1)$, define the random variable 
$$X_{t}=(1-t)X_0+tX_1$$ and set $p_t^{\pi}=\text{Law}(X_t)$, i.e., $p_t^{\pi}$ is the distribution of $X_{t}$. The general goal of the flow matching problem is to find a velocity field $v:[0,1]\times\mathbb{R}^{D}\rightarrow \mathbb{R}^{D}$ such that transporting the mass of $p_0$ along it generates the sequence $\{p_{t}^{\pi}\}$, i.e., for $t\in [0,1]$
$$\frac{\partial p_t^{\pi}(x_t)}{\partial t}=-\text{div}\big(p_t^{\pi}(x_t)v_t(x_t)\big).$$

The idea of FM underlies state-of-the-art generative AI methods \cite{esser2024scaling}. In general, the problem may not have a unique solution. Therefore, one typically seeks one \textit{specific solution}
$$v^{\pi}_t(x_{t}):=\mathbb{E}\big[X_1-X_0|X_t=x_{t}\big],$$ which is relatively easy to find numerically in practice by solving a regression problem \cite[Formula 1]{liu2023flow}, and it satisfies the required properties above \cite[Section 2.2]{liu2023flow}. Thus, the forward FM problem is associated precisely with finding this velocity field $v^{\pi}$ given the plan $\pi$.

Recently, the \underline{\textbf{inverse problem}} of FM has begun to attract interest. Given $p_0,p_1$, as well as $\{p_t^{\pi}\}_{t\in (0,1)}$ and $v^{\pi}$ obtained via FM with some plan $\pi\in \Pi(p_0,p_1)$, the task is to find the plan $\pi$ itself. Readers interested in the practical motivation for solving the inverse problem are referred to \cite{huang2024flow}, where numerically efficient generative models based on FM are developed by solving the inverse problem. Note that in practice, only $p_0$ and $v^{\pi}$ are considered as input data, since the remaining $\{p_{t}^{\pi}\}_{t\in (0,1]}$ are reconstructed using the continuity equation.

It is still theoretically unknown whether the solution $\pi$ to the inverse problem is unique. Before studying this question, let us note one nuance. The velocity field $v^{\pi}_t$ should be considered only up to values outside the support of the distribution $p_t^{\pi}$, as they do not play a role in the continuity equation. Consider the case $D=1$.

\begin{theorem}[On the Uniqueness of the Solution to the Inverse FM Problem in the One-Dimensional Case] Consider distributions $p_0,p_1\in\mathcal{P}_{\rm exp}(\mathbb{R}^{D})$. Let $\pi,\pi'\!\in\!\Pi(p_0,p_1)$ be two transport plans between $p_0,p_1$. If for all $t\in[0,1]$ $p_{t}^{\pi}\!=\!p_{t}^{\pi'}$, then $\pi\!=\!\pi'$.
\label{theorem-1}
\end{theorem}
\begin{proof}[Proof] Consider $$\phi_{\pi}(\xi_0,\xi_1)\!:=\!\mathbb{E}e^{i \xi_0 X_0+i \xi_1 X_1 },$$ which is the characteristic function of the distribution $\pi$. We will show that it is uniquely determined by the sequence $p_{t}^{\pi}$, $t\in [0,1]$. Indeed, knowing $p_{t}^{\pi}=\text{Law}(X_t)$ we uniquely recover 
$$\mathbb{E}e^{i\xi_{t}X_{t}}=\mathbb{E}e^{i\xi_{t}\!\big[\!(1-t)X_0+tX_1\!\big]}$$ for any $t\in [0,1]$ and $\xi_{t}\in \mathbb{R}$. By varying $t$ and $\xi_{t}$ within the admissible range, we can obtain all expressions of the form $\mathbb{E}e^{i \xi_0 X_0+i \xi_1 X_1 }$ with $\xi_0,\xi_1\in\mathbb{R}_{+}$, i.e., recover $\phi_{\pi}(\xi_0,\xi_1)$ on $\mathbb{R}_{+}\!\!\times\!\mathbb{R}_{+}$. Since $p_0,p_1$ have an exponential moment, $\phi_{\pi}$ is analytic in a neighborhood of $(0,0)$ and extends uniquely to $\mathbb{R}^{2}$ \citep[Corollary 2.3.8]{krantz2001function}. Thus, from the sequence $\{p_t^{\pi}\}$ we uniquely recover $\phi_{\pi}$ and $\pi$. Since $p_{t}^{\pi}=p_{t}^{\pi'}$ for all $t$, we conclude $\pi=\pi'$, as their characteristic functions are determined identically by the sequence $\{p_t^{\pi}\}$.
\end{proof}
\vspace{-1mm}In the proof, we relied on information about the sequence of distributions $p_{t}^{\pi}$ and did not use the data about the field $v^{\pi}$ from FM. It might seem that knowing the vector field is redundant. However, when attempting to generalize the proof to the multivariate case ($D>1$), a limitation is revealed: the characteristic function for $\pi$ is recovered only at points $$(\xi_0,...,\xi_0,\xi_1,...,\xi_1) \in \mathbb{R}^D\times\mathbb{R}^D,$$ where each $\xi_0,\xi_1\in\mathbb{R}_{+}$ is repeated $D$ times. This subset, being embedded in a two-dimensional plane of the space $\mathbb{R}^D\times\mathbb{R}^D$, is not open, which does not guarantee a unique extension of the characteristic function to the entire space. An illustrative example of such ambiguity in recovering $\pi$ is the Gaussian case.

Let $p_0=\mathcal{N}(\mu_0,\Sigma_0)$ and $p_1=\mathcal{N}(\mu_1,\Sigma_1)$ be Gaussian distributions with means $\mu_0, \mu_1\in \mathbb{R}^{D}$ and covariances $0\prec\Sigma_0,\Sigma_1\in\mathbb{R}^{D\times D}$, respectively. Consider the plan $$\pi=\mathcal{N}\!\left(\!\left(\begin{smallmatrix}\mu_0 \\ \mu_1\end{smallmatrix}\right)\!, \left(\begin{smallmatrix}\Sigma_0 & S \\ S^\top & \Sigma_1\end{smallmatrix}\right)\!\right)\in \Pi(p_0,p_1),$$ where $S\in \mathbb{R}^{D\times D}$. 
If $(X_0,X_1)\sim\pi$ and $X_{t}:=(1-t)X_0+tX_1$, then the triple $(X_0,X_t,X_1)$ has a Gaussian distribution. Furthermore,
$\text{Law}(X_{t})=p_{t}^{\pi}=\mathcal{N}(\mu_{t},\Sigma_t),$ where $$\mu_{t}:=(1-t)\mu_0+t \mu_1\qquad\Sigma_{t}:=(1-t)^{2}\Sigma_0+t^{2}\Sigma_1+t(1-t)(S+S^{\top}).$$

\vspace{1mm}Note that the sequence $\{p_t^{\pi}\}$ is the same for all Gaussian plans $\pi$ that share the symmetric part of the matrix $S$, i.e., $\frac{1}{2}(S + S^{\top})$. In the one-dimensional case ($D\!=\!1$, $S\!\in\!\mathbb{R}^{1\times 1}$), we always have $S\!+\!S^{\top}\!=\!2S$, so from $p_{t}^{\pi}$ we uniquely obtain $S$ and the plan $\pi$. However, for $D>1$, one can construct Gaussian $\pi,\pi'\in \Pi(p_0,p_1)$, for which $$S\!+\!S^{\top}\!=\!S'\!+\!(S')^{\top},$$ but $S\neq S'$. From this observation, it follows that information about the sequence $p_{t}^{\pi}$ alone is insufficient to precisely find the original $\pi$. Note that the additional assumption of having information about $v^{\pi}$ changes the situation.

\begin{theorem}[On the Uniqueness of the Gaussian Solution to the Inverse FM Problem in the Multivariate Gaussian Case] Let $p_0=\mathcal{N}(\mu_0,\Sigma_0)$ and $p_1=\mathcal{N}(\mu_1,\Sigma_1)$ be two Gaussian distributions on $\mathbb{R}^{D}$ and let $\pi,\pi'\!\in\!\Pi(p_0,p_1)$ be Gaussian plans with cross-covariance components $S$ \& $S'$, respectively. 
If for $t\!=\!0$ we have $v_{0}^{\pi}\!=\!v_{0}^{\pi'}$, then it follows that $\pi\!=\!\pi'$.
\end{theorem}

\begin{proof}[Proof]
 Note that $v^{\pi}_{0}(x_0)\!=\!\mathbb{E}\big[X_1\!-\!X_0|X_0\!=\!x_0\big]\!=\!\mathbb{E}[X_{1}|X_{0}\!=\!x_0]-x_0=\mu_{1}+S^{\top}\Sigma_{0}^{-1}(x_{0}\!-\!\mu_{0})-x_0$.
Similarly, we get $v^{\pi'}_{0}(x_0)=\mu_{1}+(S')^{\top}\Sigma_{0}^{-1}(x_{0}\!-\!\mu_{0})-x_0$. Equating these linear functions in $x_0$, we notice that $S^{\top}=(S')^{\top}$, which implies $\pi=\pi'$.
\end{proof}

\vspace{-1mm}The theorem shows that in the Gaussian case, to solve the inverse problem, it is sufficient to know only the initial velocity $v_{0}^{\pi}$ — the intermediate values ${p_{t}^{\pi}}$ and $v_{t}^{\pi}$ for $t\!>\!0$ are not needed. However, this theorem considers only solutions $\pi$ belonging to the \textit{Gaussian} class. It is still unknown whether other solutions outside this class exist.

\vspace{1mm}
\noindent\textbf{Conclusion.} We have considered important cases of the inverse flow matching problem, but the general conditions for the uniqueness of the solution in the multivariate problem remain an open question. Its solution is important both for theory and for practice — for example, for rigorous justification of modern FM-based generative AI methods.

\makeatletter
\renewcommand{\@makefnmark}{}
\makeatother

\centering
\small
\bibliographystyle{plain}
\bibliography{Bibliography}

\end{document}